\documentclass[twocolumn]{article}

\usepackage{amsmath}
\usepackage{amsthm}
\usepackage{amssymb}
\usepackage{amsfonts}
\usepackage{physics}
\usepackage{graphicx} 
\usepackage{booktabs}
\usepackage{url}
\usepackage{bm}
\usepackage{mathtools}
\usepackage{natbib}
\usepackage{color}
\usepackage{multirow}
\usepackage{tikz}

\theoremstyle{plain}
\newtheorem{theorem}{Theorem}

\newtheorem{proposition}[theorem]{Proposition}

\theoremstyle{definition}
\newtheorem{definition}[theorem]{Definition}

\newtheorem{remark}[theorem]{Remark}

\newcommand{\Span}[1]{\mathrm{span}({#1})}
\newcommand{\videal}[1]{\mathcal{I}({#1})}
\newcommand{\videalc}[1]{\mathcal{I}_{\mathbb{C}}({#1})}

\begin{document}

\title{Vanishing Component Analysis with Contrastive Normalization}

\date{}

\author{Ryosuke Masuya\footnote{Graduate School of Science, Tokyo Metropolitan University (email: \texttt{masuya-ryosuke@ed.tmu.ac.jp})} 
\and Yuichi Ike\footnote{Graduate School of Information Science and Technology, The University of Tokyo (email: \texttt{ike@mist.i.u-tokyo.ac.jp})} 
\and Hiroshi Kera\footnote{Graduate School of Engineering, Chiba University (email: \texttt{kera@chiba-u.jp})}}

\maketitle

\begin{abstract}
    Vanishing component analysis (VCA) computes approximate generators of vanishing ideals of samples, which are further used for extracting nonlinear features of the samples.
    Recent studies have shown that normalization of approximate generators plays an important role and different normalization leads to generators of different properties.
    In this paper, inspired by recent self-supervised frameworks, we propose a contrastive normalization method for VCA, where we impose the generators to vanish on the target samples and to be normalized on the transformed samples.
    We theoretically show that a contrastive normalization enhances the discriminative power of VCA, and provide the algebraic interpretation of VCA under our normalization.
    Numerical experiments demonstrate the effectiveness of our method. 
    This is the first study to tailor the normalization of approximate generators of vanishing ideals to obtain discriminative features. 
\end{abstract}

\section{INTRODUCTION}

Exploring the geometry of data is a common task across various fields, such as machine learning, computer vision, and systems biology.
There, it is essential to find a non-linear structure of data that is smaller than the original feature space.
One of the methods for extracting non-linear features from data is to describe the data as an algebraic set\footnote{An algebraic set here refers to a point set that can be described as the solutions of a polynomial system.}.
The method is based on the algebraic set assumption that asserts that data in $\mathbb{R}^n$ is located in an algebraic set whose dimension is much smaller than $n$.

Given a point set $X\subset \mathbb{R}^n$, the vanishing ideal $\videal X$ is the set of polynomials as follows:
\begin{align*}
    \videal X = \{g\in \mathbb{R}[x_1,\ldots,x_n]\mid g(\bm{x})=0 \ \mathrm{for \ all} \ \bm{x}\in X\}
\end{align*}
It is well-known that $\videal X$ has a finite set of generators.
However, an exact vanishing polynomial may result in a corrupted model that overfits the noisy data and be far from the actual structure.
To avoid overfitting the noisy data, we consider a generator that approximately takes $0$ on $X$, which is called an approximate vanishing generator.
In the last decade, the computation of approximate vanishing generators of $\videal X$ has been extensively studied in computer algebra and machine learning and exploited in applications such as signal processing and computer vision~\citep{torrente2009application,livni2013vanishing,hou2016discriminative,kera2016noise,kera2016vanishing,shao2016nonlinear,iraji2017principal,wang2018nonlinear,wang2019polynomial,antonova2020analytic,karimov2020algebraic}. 

Among the algorithms that provide approximate vanishing generators, the vanishing component analysis (VCA;~\cite{livni2013vanishing}) computes approximate vanishing generators without monomial orderings. 
Thus, the effect of the choice of a monomial ordering on the results need not be considered; otherwise, several runs with different monomial orderings are required to ease the effect~\citep{laubenbacher2004computational,laubenbacher2009computer}.

The VCA is further generalized to the normalized vanishing component analysis (normalized VCA;~\citeauthor{kera2019spurious}~\citeyear{kera2019spurious}), which computes approximate vanishing generators satisfying a given normalization.
Normalization such as coefficient normalization~\citep{kera2019spurious}, gradient normalization~\citep{kera2020gradient}, or gradient-weighted normalization~\citep{kera2022border} provides approximate vanishing generators with various properties that reflect geometric intuition.

In computation algorithms for approximate vanishing basis, low-degree polynomials of the computed polynomials play important roles in discovering an algebraic set of $X$. 
The intersection of algebraic sets represented by low-degree polynomials is considered a positive-dimensional non-linear structure of the data.
Such an algebraic set is thought to describe the geometry behind the data.
However, such a non-linear structure is not necessarily class-discriminative  as it is not aware of other classes.

In this paper, we propose a new normalization method for VCA, contrastive normalization, which enhances the specificity of the generators to the class of given samples. Unlike existing methods, our framework constructs discriminative generators even without using the samples from other classes, which allows us to exploit generators for single-class classification or anomaly detection. 
Our strategy is to compute generators that approximately vanish for a given data $X$, and are normalized on another data $Y$ that is designed to be similar to $X$ but at the same time belongs to a different class, thereby focusing the generator computation on the discriminative features. 
For example, when $X$ is a set of hand-written images of digit 0, $Y$ can be the samples from other digits.
This forces the generators to be discriminative within the \textit{hand-written-digit space} (see Fig.~\ref{fig:feature_space}). By contrast, if $Y$ is a set of random images, the generators can have poor discriminability (e.g., only distinguishing hand-written digits from other random images). We theoretically and empirically show that this is the case and that the design of $Y$ has a great impact on the discriminability.
Furthermore, inspired by recent self-supervised frameworks ~\citep{golan2018deep,bergman2020classification}, our normalization also works by generating good $Y$ from $X$, which broadens the applications of generators to the tasks where one can only get access to a single class (e.g., anomaly detection).

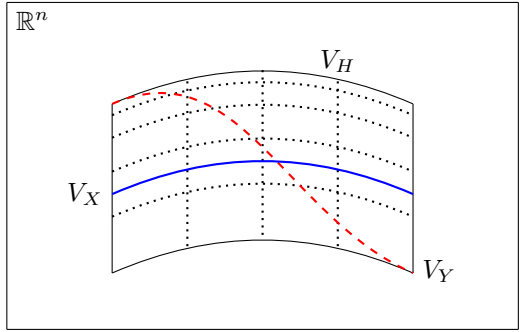
\begin{figure}[t]
\begin{center}
\begin{tikzpicture}[xscale=2,yscale=1.5]
    \draw (0,0.3) to [out=30,in=150] (2,0.3);
    \draw[dotted,thick] (0,0.2) to [out=30,in=150] (2,0.2);
    \draw[dotted,thick] (0,0) to [out=30,in=150] (2,0);
    \draw[dotted,thick] (0,-0.3) to [out=30,in=150] (2,-0.3);
    \draw[dotted,thick] (0,-0.7) to [out=30,in=150] (2,-0.7);
    \draw (0,-1.2) to [out=30,in=150] (2,-1.2);
    
    \draw (0,-0.5) node[left]{$V_X$};
    \draw (2,-1.0) node[below right]{$V_Y$};
    \draw (1.5,0.5) node[above]{$V_H$};
    \draw (-0.7,1.2) node[below right]{$\mathbb{R}^n$};

    \draw (0,0.3) to  (0,-1.2);
    \draw (2,0.3) to  (2,-1.2);
    \draw[dotted,thick] (1,0.6) to  (1,-0.9);
    \draw[dotted,thick] (0.5,0.5) to  (0.5,-1.0);
    \draw[dotted,thick] (1.5,0.5) to  (1.5,-1.0);

    \draw (2.7,1.2) to  (-0.7,1.2);
    \draw (2.7,-1.7) to  (-0.7,-1.7);
    \draw (-0.7,1.2) to  (-0.7,-1.7);
    \draw (2.7,1.2) to  (2.7,-1.7);
    
    \draw[red,thick,dashed] (0,0.3) to [out=30,in=150] (2,-1.2);
    \draw[blue,thick] (0,-0.5) to [out=30,in=150] (2,-0.5);
    
\end{tikzpicture}
\caption{The hand-written-digit space and its subspaces associated with classes.
Here, $V_H$ (surface) denotes the hand-written-digit space $\subset\mathbb{R}^n$, and $X,Y\subset V_H$ denote different classes. 
The different classes are contained in subspaces $V_X$ (blue horizontal curve) and $V_Y$ (red dashed vertical curve) of $V_H$, respectively. 
If a basis vanishing polynomial for $X$ is normalized on $Y$,  the polynomial is expected to only vanish $V_X$ in $V_H$.     \label{fig:feature_space}}
\end{center}
\end{figure}

Our contributions are summarized as follows.
\begin{enumerate}
    \item 
    We propose contrastive normalization for VCA, which is the first class-discriminative and self-supervised normalization that is accompanied by an algebraic and geometric interpretation of the computed generators. 
    In particular, we show that an ideal quotient occurs by extending the field of real numbers to that of complex numbers. 

    \item 
    We prove the importance of the choice of samples $Y$ to which generators are normalized. 
    In particular, we prove that, when $Y$ is set to random samples, the generators lose the discriminability in high probability. In addition, we empirically show that good $Y$ can be generated in a self-supervised manner. 
    

    \item
    Exploiting the self-supervised nature of the proposed framework, we apply the approximate basis computation to anomaly detection for the first time.
    The results support our theoretical arguments and also show the effectiveness of contrastive normalization.
\end{enumerate}

\section{RELATED WORK}

The approximate computation algorithms of generators of vanishing ideals have been developed first in computer algebra and then imported to machine learning~\citep{abbott2008stable,heldt2009approximate,fassino2010almost,limbeck2013computation,livni2013vanishing,kiraly2014dual,kera2018approximate,kera2022border,wirth2022conditional,wirth2022approximate}. In computer algebra, most algorithms are based on Buchberger--M\"oller algorithm~\citep{moller1982construction} and its variants~\cite{kehrein2005charactorizations,kehrein2006computing}. Particularly, several algorithms were dedicated to handling perturbed samples~\citep{abbott2008stable,heldt2009approximate,fassino2010almost,limbeck2013computation}, which are more practical for data-centric applications such as machine learning. However, although there are a few exceptions~\citep{sauer2007approximate,hashemi2019computing}, these algorithms depend on the monomial order. Different monomial orders can give different sets of generators, which becomes a problem outside computer algebra. 

In machine learning, \cite{livni2013vanishing} proposed a monomial-order-free algorithm, VCA, and theoretically showed an advantage of the use of generators of vanishing ideals for the classification task. VCA has been followed by several variants~\citep{kiraly2014dual,hou2016discriminative,kera2018approximate}; however, these algorithms, including VCA, suffer from the spurious vanishing problem~\citep{kera2019spurious}---\textit{any nonvanishing polynomials become approximately vanishing polynomials by scaling if no normalization is used}. 
To resolve this issue, \cite{kera2019spurious} proposed normalized VCA, which incorporates a normalization in VCA. 
Since then, several normalizations have been proposed and shown to be superior to the conventional coefficient normalization~\citep{kera2020gradient,kera2021monomial,kera2022border}. 

Our study also proposes a new normalization. In contrast to the existing ones, our normalization---contrastive normalization---focuses on enhancing the class-discriminability of generators. 
While \cite{kiraly2012regression,hou2016discriminative} also consider class-discriminative generators, there are two critical differences.
First, inspired by recent self-supervised frameworks in machine learning~\citep{golan2018deep,bergman2020classification}, our framework is designed to work even when there are no accessible samples from other classes.
This allows us, for the first time, to apply approximate generators to anomaly detection, where only a single class (i.e., the \textit{normal} class) is given. 
Second, we provide the algebraic and geometric interpretation of the output generators while \citep{kiraly2012regression,hou2016discriminative} do not.

\section{PRELIMINARIES}

Throughout this paper, we focus on the polynomial ring $\mathcal{R}_n=\mathbb{R}[x_1,\ldots,x_n]$, where $x_i$ is the $i$-th indeterminate.


First, we introduce the vanishing ideal of a point set.

\begin{definition}[Vanishing Ideal]\label{def:vanishing-ideal}
  The \emph{vanishing ideal} $\videal{X}\subset\mathcal{R}_n$ of a subset $X$ of $\mathbb{R}^n$ is the set of polynomials that vanish for any point in $X$:
  \begin{align*}
    \videal X = \{g\in \mathbb{R}[x_1,\ldots,x_n]\mid g(\bm{x})=0 \ \mathrm{for \ all} \ \bm{x}\in X\}.
\end{align*}
\end{definition}

\begin{definition}[Evaluation Vector and Evaluation Matrix]
  Given a point set $X=\{\bm{x}_1,\bm{x}_2,\ldots,\bm{x_{|X|}}\}$, the \emph{evaluation vector} of a polynomial $h$ is defined as
  \begin{align*}
    h(X) & =\begin{pmatrix}h(\boldsymbol{x}_{1}) & h(\boldsymbol{x}_{2}) & \cdots & h(\boldsymbol{x}_{|X|})\end{pmatrix}^{\top}\in\mathbb{R}^{|X|},
\end{align*}
where $|\cdot|$ denotes the cardinality of a set.
For a set of polynomials $H=\{h_1,h_2,\ldots,h_{|H|}\}$, its \emph{evaluation matrix} is $H(X) = \left(h_{1}(X)\,h_{2}(X)\,\cdots\,h_{|H|}(X)\right)\in\mathbb{R}^{|X|\times |H|}.$
\end{definition}
We represent a polynomial $h$ by its evaluation vector $h(X)$.
Then the product and the weighted sum of polynomials are represented as linear algebraic operations as follows.
Let $H=\{h_1,h_2, \dots ,h_{|H|}\}$ be a set of polynomials.
The product of $h,h'\in H$ is represented as $h(X)\odot h'(X)$, where $\odot$ denotes the Hadamard product.
The weighted sum $\sum_{i=1}^{\vert H \vert}w_ih_i$, where $w_i\in \mathbb{R}$, is represented as $\sum_{i=1}^{\vert H \vert}w_ih_i(X)$.
We denote $\sum_{i=1}^{\vert H \vert}w_ih_i$ by $H\bm{w}$ with $\bm{w}=\left(w_1 \ \cdots \ w_{\vert H\vert}\right)^T$.
Similarly, we denote the product between a polynomial set $H$ and a matrix $W=(\bm{w}_1\ \bm{w}_2\ \cdots \bm{w}_{s})\in\mathbb{R}^{|H|\times s}$ by $HW:=\{H\bm{w}_1,H\bm{w}_2, \dots,H\bm{w}_{s}\}$. Note that $(H\bm{w})(X) = H(X)\bm{w}$ and $(HW)(X) = H(X)W$. 
We define 
$\Span{H}=\{\sum_{h\in H}a_hh\mid a_h\in\mathbb{R}\} \subset \mathcal{R}_n$ and $\langle H \rangle=\{\sum_{h\in H}f_hh\mid f_h\in\mathcal{R}_n\} \subset \mathcal{R}_n$. 
We call the former the \textit{span} of $H$ and the latter the ideal \textit{generated} by $H$.

We consider the following approximate vanishing polynomials.

\begin{definition}[$\epsilon$-vanishing Polynomial]
  A polynomial $g$ is an \emph{$\epsilon$-vanishing polynomial} for a point set $X$ if $\Vert g(X)\Vert\leq \epsilon$, where $\Vert \cdot\Vert$ denotes the Euclidean norm; otherwise, $g$ is an $\epsilon$-nonvanishing polynomial.
\end{definition}
\section{PROPOSED METHOD}\label{sec:proposed_method}


\cite{kera2019spurious} proposed a basis construction algorithm of vanishing ideals, called the normalized VCA.
The normalized VCA constructs vanishing polynomials for a point set under given normalization.
They used coefficient and gradient normalization of vanishing polynomials, which overcome the spurious vanishing problem in vanishing polynomials.
We propose to require vanishing polynomials for a point set to be normalized on another point set.

\subsection{Algorithm}\label{sec:algorithm}

The input to our algorithm is two point sets $X,Y \subset \mathbb{R}^{n}$ and error tolerance $\epsilon\ge 0$. 
The algorithm outputs a basis set $G$ of $\epsilon$-vanishing polynomials and a basis set of $\epsilon$-nonvanishing polynomials $F$. 
The algorithm proceeds from degree-0 polynomials to those of higher degrees. 
For each $t$, a set of degree-$t$ $\epsilon$-vanishing polynomials $G_t$ and a set of degree-$t$ $\epsilon$-nonvanishing polynomials $F_t$ are generated. 
We use notations $F^{t}=\bigcup_{\tau=0}^t F_{\tau}$ and $G^{t}=\bigcup_{\tau=0}^t G_{\tau}$.
For $t=0$, $F_0 = \{m\}$ and $G_0 = \emptyset$, where $m\ne 0$ is a constant polynomial. 
At each degree $t\ge 1$, the following procedures (\texttt{Step 1}, \texttt{Step 2}, and \texttt{Step 3}) are conducted\footnote{For ease of understanding, we describe the procedures in the form of symbolic computation, but these can be numerically implemented (i.e., by matrix-vector calculations)}.

  \paragraph*{Step 1: Generate a set of candidate polynomials.}
  Pre-candidate polynomials of degree $t$ for $t> 1$ are generated by multiplying nonvanishing polynomials across $F_1$ and $F_{t-1}$:
  \begin{align*}
      C_t^{\mathrm{pre}} = \{pq \mid p\in F_1, q\in F_{t-1}\}.
  \end{align*}
  At $t=1$, $C_1^{\mathrm{pre}}=\{x_1,x_2,\dots ,x_n\}$. The candidate basis is then generated through orthogonalization:
  \begin{align}\label{eq:orthogonalization}
       C_{t} &= C_{t}^{\mathrm{pre}} - F^{t-1}(F^{t-1}(X))^{\dagger}C_{t}^{\mathrm{pre}}(X),
   \end{align}
   where $(\cdot)^{\dagger}$ is the pseudo-inverse of a matrix.
  \paragraph*{Step 2: Solve a generalized eigenvalue problem.}{
  We solve the following generalized eigenvalue problem: 
  \begin{align}\label{eq:gep}
      C_t(X)^{\top}C_t(X)V = \mathfrak{N}(C_t)V\Lambda,
  \end{align}
  where $V$ is the matrix that has generalized eigenvectors $\boldsymbol{v}_1,\boldsymbol{v}_2,\dots ,\boldsymbol{v}_{|C_t|}$ for its columns, $\Lambda$ is the diagonal matrix with generalized eigenvalues $\lambda_1,\lambda_2,\dots ,\lambda_{|C_t|}$, and $\mathfrak{N}(C_t)\in\mathbb{R}^{|C_t|\times |C_t|}$ is the normalization matrix whose $(i,j)$-th entry is $c_i(Y)^Tc_j(Y)/\vert Y\vert$ with $c_i,c_j\in C_t$.
  }
  \paragraph*{Step 3: Construct sets of basis polynomials.}{
  Basis polynomials are generated by linearly combining polynomials in $C_t$ with $\{\boldsymbol{v}_1,\boldsymbol{v}_2,\dots ,\boldsymbol{v}_{|C_t|}\}$:
  \begin{align*}
      G_t &= \{C_t\boldsymbol{v}_i\mid \sqrt{\lambda_i} \le \epsilon\}, \\
      F_t &= \{C_t\boldsymbol{v}_i\mid \sqrt{\lambda_i} > \epsilon\}.
  \end{align*}
  If $|F_t| = 0$, the algorithm terminates with output $G=G^t$ and $F=F^t$.
  }

\begin{remark}
    By solving a generalized eigenvalue problem in  \texttt{Step2}, we have $\Vert g(Y)\Vert^2/\vert Y\vert=1$ for any $g\in G$.
  Hence, $g$ is a $\sqrt{\vert Y\vert}$-nonvanishing polynomial for $Y$.
  In particular, if $\epsilon<\sqrt{\vert Y\vert}$, then $g$ is $\epsilon$-vanishing for $X$ and $\epsilon$-nonvanishing for $Y$.
\end{remark}

We denote the algorithm for two sets $X,Y$ by $\texttt{VCA}(X,Y)$.

\paragraph{The choice of $Y$.}
Here, for a data set $X$, we consider the choice of $Y$ such that basis polynomials obtained by $\texttt{VCA}(X, Y)$ have discriminability in a target space containing $X$. 
We first remark that, if basis polynomials are constant $0$ on the target space, then they have no discriminability. 
In our method, by using another data set $Y\not\subset X$ in the same space, $\texttt{VCA}(X, Y)$ gives basis polynomials that vanish on $X$ and not on $Y$, which means that they are not constant $0$ on the target space.
Hence, it is important to choose a suitable point set $Y$ from the target space for discriminability. 
Furthermore, as explained as follows, we can naturally consider the choice of $Y$ in self-supervised frameworks.

In self-supervised learning, some methods learn given data and those transformed by suitable transformations.
In particular, using transformations with features of data enables a classifier to learn high-quality data.
For instance, for hand-writing data, image processing transformations are considered (rotation by $0,90,180,270$ degrees). 
Self-supervised frameworks also enable $\texttt{VCA}(X,Y)$ to learn high-quality data.
In particular, we set $Y$ to an appropriately transformed $X$.

\subsection{Our Strategy for Anomaly Detection}
In this section, we propose our anomaly detection method by using $\texttt{VCA}(X,Y)$ and the idea of GOAD~\citep{bergman2020classification}, which is a deep anomaly detection method.
Similarly to GOAD, our method leans the normal data and transformations.
In the following, we describe the outline of our strategy.
First, we transform $X\subset\mathbb{R}^n$ to $X_i\subset\mathbb{R}^n$ using a transformation $T_i: \mathbb{R}^n\to\mathbb{R}^n$ for $i=1,\ldots,M$.
We then compute $\texttt{VCA}(X_i,Y_i)$, where $Y_i = \bigcup_{j\ne i}X_j$.
The obtained set of generators is denoted by $G_{i}=\{g_{1}^{(i)},\ldots,g_{|G_{i}|}^{(i)}\}$.
By our construction, $G_i(\boldsymbol{x})$ for new sample $\bm{x}\in\mathbb{R}^n$ is expected to take small values if $\boldsymbol{x}\in V(G_i)$.
Moreover, an algebraic set $V(G_i)=\bigcap_{j=1}^{|G_{i}|} V(g_{j}^{(i)})$ is expected to discriminate $X_i$ from the other classes.

We define the classifier that predicts transformation $T_i$ given a transformed point as follows:
\begin{align*}
    P_i(\bm{x})=\frac{e^{-\Vert G_i(T_i\bm{x}) \Vert^2}}{\sum_{k=1}^M e^{-\Vert G_k(T_i\bm{x}) \Vert^2}}.
\end{align*}
This means that if $P_i(\bm{x})\approx 1$, the sample $T_i\bm{x}$ approximately resides in $V(G_i)$ and not in $V(G_j)$. 

By assuming independence between transformations, the probability of $\bm{x} \in X$ is the product of $P_i$'s. 
Therefore, an anomaly score of $\bm{x}\in X$ is induced from the probability that all transformed samples are in their respective algebraic sets as follows:
\begin{align}\label{eq:score}
    \mathrm{Score}(\bm{x})=-\log \prod_{i=1}^M P_i(\bm{x})=-\sum_{i=1}^M\log  P_i(\bm{x}).
\end{align}
A high score for $\bm{x}$ indicates that $\bm{x}$ is anomalous.

The feature vectors $G_i(\bm{x})$'s for $\bm{x}$ are not strictly zero vectors.
Therefore, we replace the zero vectors by centers $\bm{c}_i$, which are given by the average feature over the training set for every transformation, i.e., $\bm{c}_i=\frac{1}{\vert X \vert}\sum_{\bm{x}\in X}G_i(\bm{x})$.
Also, GOAD added a small regularizing constant $\epsilon$ to the probability of each transformation to avoid uncertain probabilities.
We added $\epsilon$ similarly:
\begin{align*}
    P_i(\bm{x})=\frac{e^{-\Vert G_i(T_i\bm{x})-\bm{c}_i \Vert^2}+\epsilon}{\sum_{k=1}^M e^{-\Vert G_k(T_i\bm{x})-\bm{c}_k \Vert^2}+M\epsilon}.
\end{align*}
Note that, in GOAD, $G_i$ is implemented by a deep neural network. 
Furthermore, we omit algebraic sets not used in anomaly detection as follows.
We replace the basis set $G_i$ by $\bm{w}G_i=(w_1g_1^{(i)},\ldots,w_{\vert G_i\vert}g_{\vert G_i\vert}^{(i)})$, where $\bm{w}=(w_1,\ldots,w_{\vert G_i \vert})\in \mathbb{R}^n$.
Updating $\bm{w}$ by optimizing the score Eq.~\eqref{eq:score}, we pick up effective algebraic sets in anomaly detection.

\section{THEORETICAL ANALYSIS}\label{sec:theoretical_analysis}

In this section, we analyze the discriminative power of $\texttt{VCA}(X,Y)$ and give its algebraic interpretation.
See supplementary material for proofs of the results in this section.

\subsection{Discriminative Power of \texttt{VCA}(X,Y)}

In Section~\ref{sec:algorithm}, we described that, for data $X$ and another data $Y$ that is similar to $X$ but belongs to a different class, $\texttt{VCA}(X,Y)$ gives a nonlinear structure of $X$. 
Then, the nonlinear structure has highly discriminative power in the space where both $X$ and $Y$ belong.
In this section, we theoretically prove that setting $Y$ to a set of random samples gives a nonlinear structure of $X$ without discriminative power in the space.

For simplicity, we assume that all spaces are subspaces of $\mathbb{R}^n$ and that data $X$ is of mean $\bm{0}$.

Let $V$ be a subspace of $\mathbb{R}^n$ containing $X$, where we consider that $V$ is the feature space of $X$. 
Let $V_X\subset V$ be a subspace, which we consider as a feature space of $X$ smaller than $V$.
Since $V_X$ is the feature space of $X$, it is naturally considered that the feature space $V_X$ of $X$ is almost reproduced by $X$, namely, $V_X=\Span{X}$.
In the following, we may assume $V_X=\Span{X}$.

Let $Z$ be a set of random samples of $\mathbb{R}^n$ and let $G$ be the output of $\texttt{VCA}(X,Z)$.
We formulate that a polynomial normalized on a random sample has no discriminative power in the target space V as follows:

Let $\bm{p} \in V_X$ and $\bm{q} \in \mathbb{R}^n$ be random points.
For $g\in G$, both $\mathrm{P}[\vert g(\bm{p})\vert<t_0]$  and $\mathrm{P}[\vert g(\bm{q}) \vert>t_1]$ are high for some $t_1>t_0>0$.

Since $Z$ is a set of random samples, each coefficient of $g$ is a random variable induced by $Z$.
However, it is difficult to represent the coefficients.
To avoid this problem, we consider a polynomial that is a \lq{}\lq{}probabilistic model" for $g$ as follows:
First, recall that $g$ satisfies the following relations:
\begin{align*}
    \lVert g(X) \rVert =0 \ {\rm and} \ \frac{\lVert g(Z)\rVert^2}{\vert Z \vert}=1.
\end{align*}
Instead of $g$, we use a polynomial $g'$ satisfying the above relations with high probability.
In particular, when $g$ is of degree $1$, $g'$ can be designed by the following proposition.

\begin{proposition}\label{prop:concentration}
  Let $X\subset \mathbb{R}^n$ be a point set and let $\{\bm{u}_1,\ldots,\bm{u}_n\}$ be an orthogonal basis of $\mathbb{R}^n$ such that $\Span{X}=\Span{\bm{u}_1,\ldots,\bm{u}_k}$.
  Choose a set, $Z$, of random points satisfying that any point $\bm{p}_i\in Z$ is of the form $\bm{p}_i=\sum_{j=1}^na_{i,j}\bm{u}_j$, where $a_{i,j}$'s $( 1\leq i \leq |Z|, 1\leq j\leq n)$ are i.i.d. $\sim N(0,1)$.
  If $\bm{w}=\sum_{i=k+1}^n w_i\bm{u}_{i}$ is a random vector such that $w_i$'s are i.i.d. $\sim N(0,1/(n-k))$, 
  then, for all $\epsilon>0$, we have 
  \begin{align*}
      \begin{aligned}
          \mathrm{P}\left[ \lVert g'(X) \rVert =0 \right] &=1,\\
          \mathrm{P}\left[\left\vert \frac{\lVert g'(Z)\rVert^2}{\vert Z \vert}-1 \right\vert \geq 2\epsilon\right]
          &\leq 
           e^{-\epsilon^2N/6}+e^{-\epsilon N /3}\\
          &\quad +(N^2-N)e^{-\eta^2 |Z|/2}\\
          &\quad +N(e^{-\eta^2 |Z|/6}+e^{-\eta |Z|/3}),
      \end{aligned}
  \end{align*}
  where $g'(\bm{x})=\bm{w}^T\bm{x}$, $N=n-k$ and $\eta=\frac{\epsilon}{(1+\epsilon)N}$.
\end{proposition}

Since we use a lot of random samples, we may assume $|Z| \gg n$.
We may also assume that $n \gg k$, as the dimension of $V_X$ is much smaller than $n$.
Using a probabilistic model described in Proposition~\ref{prop:concentration}, we can prove the following theorem, which indicates the desired statement.

\begin{theorem}\label{thm:discriminative_power}
  Let $X,Y\subset \mathbb{R}^n$ be sets of points and we denote $V_X=\Span{X}$ and $V_Y=\Span{Y}$.
  Let $\{\bm{u}_1,\ldots,\bm{u}_n\}$ be an orthogonal basis of $\mathbb{R}^n$ such that 
  \begin{align*}
    \begin{aligned}
      V_X&=\Span{\bm{u}_1,\ldots,\bm{u}_k} \ \mathrm{and} \\ 
      V_X+V_Y&=\Span{\bm{u}_1,\ldots,\bm{u}_{k+m}}.      
    \end{aligned}
  \end{align*}
  Let $\bm{w}=\sum_{i=k+1}^n w_i\bm{u}_{i}$ be a random vector as in Proposition\,\ref{prop:concentration}.
  Choose random vectors $\bm{p}_X \in V_X$, $\bm{p}_Y \in V_Y$ and $\bm{p}_{\perp}\in (V_X+V_Y)^{\perp}$ such that coefficients of $\bm{u}_1,\ldots,\bm{u}_n$ are i.i.d.$\sim N(0,1)$.
  Then, for $t>0$, we have 
  \begin{align*}
      \begin{aligned}
          \mathrm{P}\left[ \left\vert g(\bm{p}_X)\right\vert=0\right] &= 1,\\
          \mathrm{P}\left[\left\vert g(\bm{p}_Y)\right\vert<t\right] &\geq 1-\frac{m}{n-k}\cdot\frac{1}{t^2},\\
          \mathrm{P}\left[\left\vert g(\bm{p}_{\perp})\right\vert\geq \sqrt{\frac{n-(k+m)}{2(n-k)}}\right] &\geq \frac{n-(k+m)}{4(n-(k+m)+2)}. 
      \end{aligned}
  \end{align*}
\end{theorem}

In Theorem~\ref{thm:discriminative_power}, we interpret $V_X+V_Y$ as the feature space $V$ described just before Proposition~\ref{prop:concentration}.
Under the assumption that $m,k<m+k=\dim V\ll n$, from Theorem~\ref{thm:discriminative_power}, we have $\frac{m}{n-k}\cdot\frac{1}{t^2}\approx 0$ and $\frac{n-(k+m)}{(n-(k+m)+2)}\approx 1$.
Namely, Theorem~\ref{thm:discriminative_power} states that $\texttt{VCA}(X,Z)$ constructs polynomials with no discriminative power in $V$.
\subsection{An Algebraic Interpretation of $\texttt{VCA}(X,Y)$}

For error tolerance $\epsilon=0$ and a set, $X$, of points,
VCA, the normalized VCA with gradient, and the normalized VCA with coefficient are known to construct bases of $\videal X$, respectively.
In this section, we consider an ideal generated by the basis polynomials of $\texttt{VCA}(X,Y)$.

\subsubsection{Basic Definitions and Notations for Ideals}
We here state some basic terminology and facts about ideals.
Let $k$ be a field and let $k[x_1,\ldots,x_n]$ be a polynomial ring, where $x_i$ is the $i$-th indeterminate.
We assume that ideals are defined in $k[x_1,\ldots,x_n]$, unless otherwise stated.

\begin{definition}
    Let $G\subset k[x_1,\ldots,x_n]$.
    Then we set
    \begin{align*}
      V(G)=\{\bm{x}\in k^n\mid f(\bm{x})=0 \ \text{{\rm for all}} \ f\in G\}.
    \end{align*} 
    We call $V(G)$ the \emph{algebraic set} defined by $G$ over $k$.
    When we emphasize a field $k$, we denote $V(G)$ by $V_k(G)$.
    
    An algebraic set $V\subset k^n$ is irreducible if $V=V_1\cup V_2$, where $V_1$ and $V_2$ are algebraic sets over $k$, then $V_1=V$ or $V_2=V$.
\end{definition}

\begin{definition}
    Let $V\subset k^n$ be an algebraic set.
    A decomposition
    \begin{align*}
        V=V_1\cup\cdots\cup V_s,
    \end{align*}
    where each $V_i$ is an irreducible variety, is called a \emph{minimal decomposition} if $V_i\not\subset V_j$ for $i\neq j$.
    Also, we call the $V_i$ the \emph{irreducible components} of $V_i$.
\end{definition}

\begin{theorem}
  Let $V\subset k^n$ be an algebraic set.
  Then, $V$ has a minimal decomposition 
  \begin{align*}
      V=V_1\cup\cdots\cup V_s.
  \end{align*}
  Furthermore, this minimal decomposition is unique up to the order in which $V_1,\ldots,V_s$ are written.
\end{theorem}

Now we define the dimension of an irreducible variety.

\begin{definition}
    Let $V\subset k^n$ be an irreducible algebraic set.
    We define 
    \begin{align*}
        \dim V = \sup \{r \mid V_0\subsetneq V_1\subsetneq \cdots \subsetneq V_r=V,\\ \ V_i \text{: irreducible over $k$}\}.
    \end{align*}
    We call $\dim V$ the \emph{dimension} of $V$.
\end{definition}

It is well-known that the dimension of an irreducible algebraic set is finite.

For a subset $X$ of $k^n$, we set 
$\mathcal{I}(X)$ 
\begin{align*}
    \mathcal{I}(X)=\{g\in k[x_1,\ldots,x_n]\mid g(\bm{x})=0 \text{ {\rm for all} } \bm{x}\in X\}.
\end{align*}
When we emphasize a field $k$, we denote $\videal X$ by $\mathcal{I}_k(X)$.

\begin{definition}
    Let $S\subset k^n$.
    We define $\overline{S}=V(\mathcal{I}(S))$.
\end{definition}

We introduce the notion of ideal quotient to describe the ideal associated with $V \setminus W$.

\begin{definition}
    Let $I$ and $J$ be ideals.
    Then we set 
    \begin{align*}
        I:J=\{g\in k[x_1,\ldots,x_n]\mid gJ \subset I\}.
    \end{align*}
    We call $I:J$ the \emph{ideal quotient} of $I$ by $J$.
\end{definition}

Note that an ideal quotient is an ideal.

\begin{proposition}[{\cite[Ch.\,4 Sect.\,4 Corollary\,11]{cox2015ideals}}]\label{prop:quotient}
    Let $V$ and $W$ be algebraic sets over $k$.
    Then we have $\videal V:\videal W=\videal {V\setminus W}$.
\end{proposition}

We are interested in describing $\videal{V(G)}$.
For that purpose, we introduce the radical ideal.
\begin{definition}
    Let $I$ be an ideal.
    The  radical ideal, $\sqrt{I}$, of $I$  is the set
    \begin{align*}
        \{f\mid f^m\in I \text{  {\rm for some integer} } m\geq 1\}.
    \end{align*}
\end{definition}

\begin{theorem}[The Strong Nullstellensatz, {\cite[Ch\,4 Sect.\,2 Theorem\,6]{cox2015ideals}}]
    Let $G\subset k[x_1,\ldots,x_n]$. 
    If $k$ is algebraically closed, then $\videal {V(G)}=\sqrt{\langle G \rangle}$.
\end{theorem}

\subsubsection{An Ideal Given by \texttt{VCA}(X,Y)}
The original \texttt{VCA}$(X)$ gives a set of generators $G_0$ that generates $\videal{X}$, i.e., $\videal{X} = \langle G_0\rangle$. In contrast, our method \texttt{VCA}$(X, Y)$ only outputs a set of vanishing polynomials $G$ that are normalized over $Y$. As a consequence, its output only generate a subset of $\videal{X}$, i.e., $\videal{X} \supset \langle G\rangle$. Here, we provide the algebraic interpretation of this subset; in particular, we show that the radical ideal of $\langle G\rangle$ generates an ideal quotient.

In the following, we consider the case $k=\mathbb{R}$ or $\mathbb{C}$.

\begin{definition}
    Let $I$ be an ideal in $\mathcal{R}_n=\mathbb{R}[x_1,\ldots,x_n]$.
    Then we define an ideal $I_{\mathbb{C}}$ in $\mathbb{C}[x_1,\ldots,x_n]$ as follows:
    \begin{align*}
        I_{\mathbb{C}}=\left\{ \sum_i f_i g_i \mid f_i\in\mathbb{C}[x_1,\ldots,x_n], g_i\in I  \right\}.
    \end{align*} 
\end{definition}

\begin{theorem}
    Let $X,Y\subset\mathbb{R}^n$ be distinct sets of points and let $G$ be the out put of $\texttt{VCA}(X,Y)$ for $\epsilon=0$.
    We put $V=V_{k}(G)$ and denote its irreducible components by $V_1,\ldots,V_s$.
    Then, for any irreducible algebraic set $W\subset k^n$ satisfying $\dim W= \min_{1\leq i \leq s} \dim V_i$ and $Y\subset W$, we have
     
        (1) $V\not\subset W$ and $W\not\subset V$.
     
        (2) $V=\overline{V\setminus W}$.
     
        (3) $\mathcal{I}_{k}(V)=\mathcal{I}_{k}(V\setminus W)=\mathcal{I}_{k}(V): \mathcal{I}_{k}(W)$.
    
    Moreover, if $k=\mathbb{C}$, then we have
    \[
        \left(\sqrt{\langle G\rangle}\right)_{\mathbb{C}}=\videalc V=\videalc {V\setminus W}=\videalc V:\videalc W.
    \]
\end{theorem}

\section{EXPERIMENTS}

\subsection{Synthetic Dataset}\label{sec:synthetic-dataset}
In this section, by using synthetic data, we confirm Theorem~\ref{thm:discriminative_power}, which states that a polynomial normalized on random samples does not have discriminative power in a target space. 
In the following, we use notation as in Theorem~\ref{thm:discriminative_power}.

Let $\bm{u}_i\in \mathbb{R}^n$ be a standard basis whose entries are all zero except the $i$-th entry that equals $1$.
Let $X=\{\bm{x}_1,\ldots,\bm{x}_{\vert X\vert}\}\subset \Span{\bm{u}_1,\ldots,\bm{u}_{k}}$ be a set of random points such that (i) for $\bm{x}_{i}\in X$, its $j$-th entries $(1\leq j\leq 5)$ are i.i.d.$\sim N(0,1)$ and (ii) $\Span{X}=\Span{\bm{u}_1,\ldots,\bm{u}_{k}}$.
Then, $V_X=\Span{X}$.
We define a target space $V=V_X+V_Y$ by $\Span{\bm{u}_1,\ldots,\bm{u}_{k+m}}$.
Let $Z=\{ \bm{p}_1, \dots, \bm{p}_{\vert Z\vert }\}$, where entries of $p_i$ are i.i.d.$\sim N(0,1)$. 
We consider the vanishing basis of degree $1$ computed by $\texttt{VCA}(X,Z)$.

In our setting with $n=105$, $k=5$, $m=5$ and $\vert X\vert=\vert Z\vert=$10,000, the $\texttt{VCA}(X,Z)$ basis set of degree $1$ consists of $18$ vanishing polynomials $g_1,\ldots,g_{18}$. 
We choose random points from $V=V_X+V_Y$ as follows:
Let $R$ be a set of 1,000 random points whose $i$-th entries $(1\leq i\leq 10)$ are i.i.d.$\sim N(0,1)$ and otherwise $0$.
We define the probability of $|g_i(\bm{p})|<t$ for $t>0$ by $P_i(t)=\vert \{\bm{p}\in R \mid |g_i(\bm{p})|<t \} \vert/|R|$.

In Fig.~\ref{fig:result-thm}, the probability $P_i(t)$ is plotted for $t = 0.24, 0.25,\ldots, 1.0$. 
We also plot the theoretical curve of $1-\frac{m}{n-k}\cdot\frac{1}{t^2}=1-0.01/t^2$, which is given as a low bound in Theorem~\ref{thm:discriminative_power}. 
As shown in Fig.~\ref{fig:result-thm}, we confirmed the justification of Theorem 6; namely, $P_i(t)$ is higher than the lower bound described in Theorem~\ref{thm:discriminative_power}.
\begin{figure}[t]
    \includegraphics[scale=0.55]{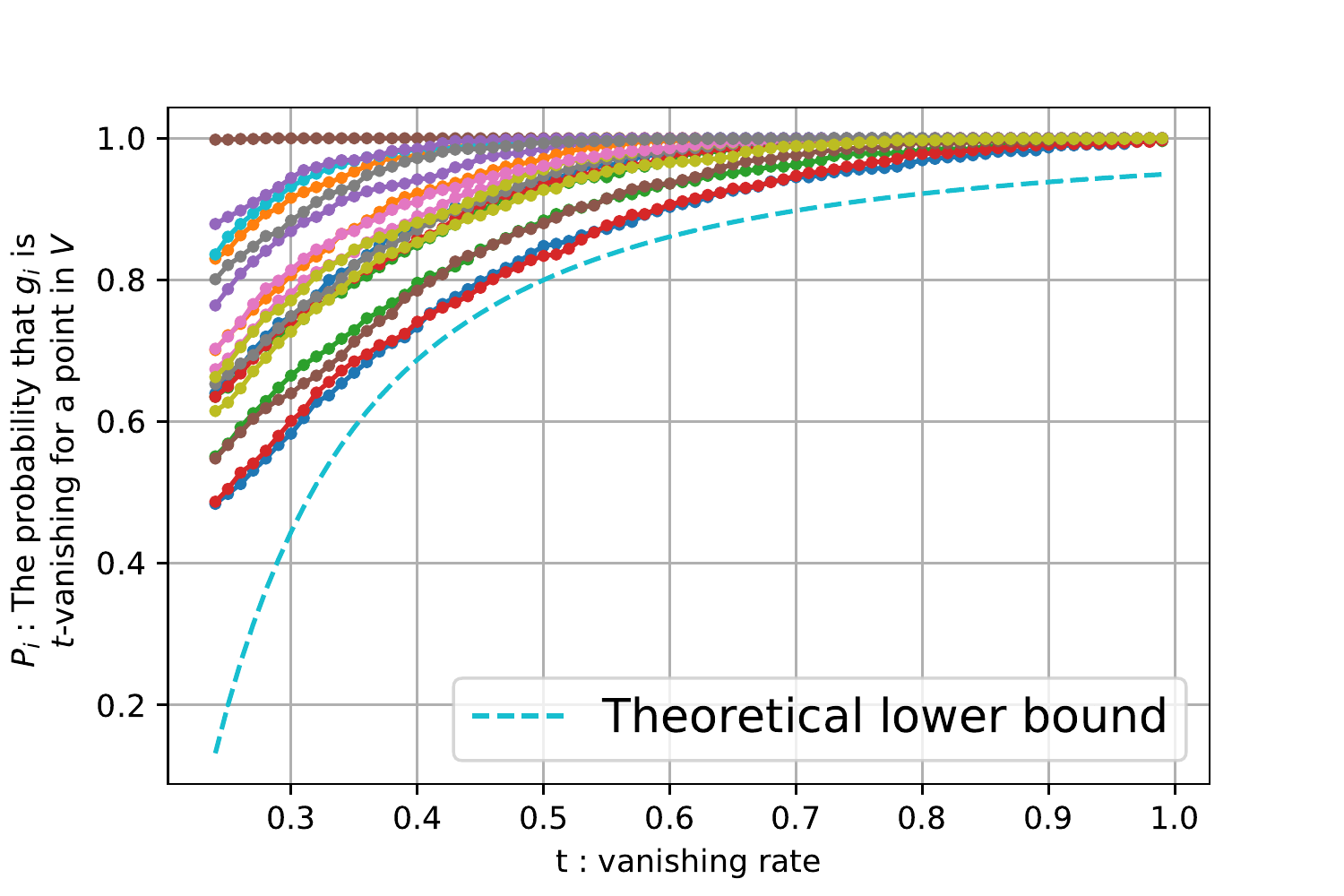}\caption{The probability that a basis polynomial is $t$-vanishing for a point in $V$.
    Basis polynomials $g_1,\ldots,g_{19}$ (solid lines in different colors) obtained by $\texttt{VCA}(X,Z)$.
    All the probabilities are higher than the theoretical value.
    This means justification on Theorem~\ref{thm:discriminative_power}.
    \label{fig:result-thm}}
\end{figure}

\subsection{Anomaly Detection for Benchmark Datasets}
In Sections~\ref{sec:proposed_method} and \ref{sec:theoretical_analysis}, we have discussed the discriminability of basis polynomials vanishing in an original data set and normalized on another data set.
In particular, we have suggested that (i)~the discriminative power of basis polynomials depends on the choice of normalizing data, and (ii)~the normalizing data should be similar to the original data but different from the original data.
Theorem~\ref{thm:discriminative_power} states that basis polynomials have no discriminability when the normalizing data set is a set of random samples.
This theorem was confirmed by using synthetic data in Section~\ref{sec:synthetic-dataset}.
In this section, we further study (ii) through experiments for anomaly detection.
In particular, we chose random affine transformed data as the normalizing data and confirmed the random affine transformed data version of Theorem~\ref{thm:discriminative_power}.
Let us describe our setting of experiments.

\paragraph*{Datasets.}
We used two standard datasets, the MNIST datasets and the FashionMNIST datasets.
In both cases, we consider that normal data set is a collection of samples labeled as 0, 2, 4, 6, 8, and anomalous data is a collection of samples labeled as 1, 3, 5, 7, 9.
We experimented with the number of training data in the following manner:
The training data was the data labeled as normal out of (i) the full 60,000 training data or (ii) the first 10,000 training data.
Also, it is considered that high-dimensional data (e.g., image data) has a low effective dimensionality in many applications.
Therefore, we first project training data and test data onto low dimensional space by a dimensionality reduction (e.g., the principal component analysis; PCA) and the change of coordinates.
The preprocessing follows \citep{kera2021monomial}.
We extracted polynomials of degree $\leq 4$.
\paragraph{Transformations.}
We used random affine transformations and rotation transformations.
In particular, rotation transformations are rotation by $0,90,180,270$ degrees.
 
\paragraph*{Hyperparamaters.}
We optimized our method using naive gradient descent with a learning rate of $1.0 \times 10^{-5}$ and $300$ epochs.

\paragraph*{Baselines.}
The baseline methods evaluated are: Vanishing Component Analysis (VCA \cite{livni2013vanishing}), Normalized Vanishing Component Analysis with gradient normalization~(nVCA, \cite{kera2020gradient}).
To show that basis polynomial algorithms for a given normal data produce basis polynomials without discriminability in a target space (e.g., the hand-written-digit space), we compared these methods in anomaly detection tasks.
Since VCA and nVCA do not require another data set, we simply design an anomaly detector using VCA and nVCA as follows:
First, following \cite{livni2013vanishing}, the feature vector $\mathcal{F}(\boldsymbol{x})$ of a data point $\boldsymbol{x}$ is defined as 
\begin{align}\label{eq:feature}
\mathcal{F}(\boldsymbol{x}) & =(\left|g_{1}(\boldsymbol{x})\right|,\dots,\left|g_{|G|}(\boldsymbol{x})\right|)^{\top},
\end{align} 
where $G=\{g_{1}, \dots ,g_{|G|}\}$ is the basis set computed for normal data points. 
Because of its construction, $\mathcal{F}(\boldsymbol{x})$ is expected to take small values if $\boldsymbol{x}$ belongs the normal data set.
Therefore, for a new data, we define a score function by $||\mathcal{F}(\boldsymbol{x})||$.
The score function indicates that $\bm{x}$ is anomalous if the score for $\bm{x}$ is high.
The results are reported in terms of AUC and shown in Table~\ref{table:classification_mnist}. 
Also, as a reference, we also choose GOAD as a self-supervised baseline.
Note that we used the official source code of GOAD as it is including the hyperparemters\footnote{\url{https://github.com/lironber/GOAD}}.
\begin{table*}[t]
    \caption{Anomaly detection accuracy (AUC score $\times 100$) in MNIST and FashionMNIST in the case when a normal class is a collection of samples labeled as $0,2,4,6,8$. Here, \textit{Rot.} denotes the rotation transformations. \textit{R.A.} denotes random affine transformations. The VCA and the nVCA resulted in no discriminability between data, which implies that basis polynomials have no discriminability in the hand-written-digit space.
    The \textit{Rot.} version of our proposed method outperforms the \textit{R.A.} version.
    This implies that, by choosing normalizing data from the target space,  we can obtain basis polynomials with discriminability.
    Note that the results of GOAD are reported as a reference.
    \label{table:classification_mnist}}
    \begin{center}
      \begin{tabular}{c|c|c|c|c|cc}
          \toprule      
        & \multicolumn{1}{l|}{Size of } 
         & \multicolumn{5}{c}{Method} \\
         Data set & training set&      & \ & &\multicolumn{2}{|c}{Ours} \\
         & & GOAD & VCA & nVCA &  \textit{Rot.} & \textit{R.A.}\\
         \midrule
          \multirow{2}{*}{MNIST} & 60,000 & \underline{76.2} & 51.8 & 49.1 & \textbf{79.1} & 54.9\\
          &10,000 & -- & 48.4 & \underline{50.7} & 50.2 & \textbf{59.1} \\
           \midrule
          \multirow{2}{*}{FashionMNIST}&60,000 & \textbf{93.5} 
          &50.2 & 51.7 & \underline{83.8} & 53.3\\
          &10,000 & -- & \underline{53.9} & 46.2 & \textbf{82.3} & 49.3 \\
           \bottomrule
          \end{tabular}
    \end{center}
\end{table*}

\paragraph{Other baselines.}
Our strategy for anomaly detection (an analogy of GOAD) is effective when we use basis polynomials with information of two datasets.
On the other hand, the effectiveness of our anomaly strategy is not able to be shown when we use basis polynomials having information of just one dataset (e.g., basis polynomials computed by VCA and nVCA), but we can apply GOAD frameworks to VCA and nVCA below.
We denote the VCA for a data set $X$ and the nVCA for $X$ by $\texttt{VCA}(X)$ and $\texttt{nVCA}(X)$, respectively.
By replacing $\texttt{VCA}(X_i,Y_i)$ in Section 4.2 by $\texttt{VCA}(X_i)$ or $\texttt{nVCA}(X_i)$, we can also perform anomaly detection using them.
We also performed these experiments to confirm that focusing on similar but different data is effective for anomaly detection.
We use AUC as a score and show the results in Table~\ref{table:classification_vca_nvca}.
Note that, in these experiments, discriminative power is not focused on.
\paragraph{Results (rotation transformations~vs.~random affine transformations).}
When we use 60,000 data points, comparing the fifth and sixth columns in Table~\ref{table:classification_mnist}, we see that anomaly detection using rotation transformations outperforms that using random affine transformations.
Increasing the size of training points enhances the rotation transformation version of our method, 
while the random affine transformation version of our method has constant discriminative power.
This shows that random affine transformed data outside the target space  does not enable basis polynomials to have discriminability in the target space.
Namely, the random affine version of Theorem~\ref{thm:discriminative_power} is confirmed.
Note that a random affine transformation is often considered a geometric transformation, and plays an important role in  manifold learning and dimensionality reduction.
However, the data after transformation is not located in the target space that the original data belongs to.
Hence, in our argument focusing on the target space, using a random affine transformed data set as the normalized part does not produce basis polynomials with discriminability.

\paragraph{Results (VCA and nVCA).}
As shown in the VCA and nVCA columns in Table~\ref{table:classification_mnist}, the computed polynomials have no discriminative power, and the results are comparable if we use 60,000 data points or 10,000 points.
This allows us to obtain the following interpretation of the discriminative power of polynomials by previous basis computation algorithms.
When we compute basis polynomials for an original data set by VCA or nVCA, the computed polynomials do not necessarily have information of other data because of VCA and nVCA algorithms.
Hence, in the setting of the experiments using VCA and nVCA described in the baselines paragraph, we can not ensure that polynomials obtained by VCA and nVCA have discriminability.

\paragraph{Results (VCA and nVCA applied to the GOAD frameworks).}
Observing the second column in Table~\ref{table:classification_vca_nvca}, VCA and nVCA applied to the GOAD frameworks achieved higher performance than simple VCA and nVCA, respectively.
Because of gradient-weighted normalization, basis polynomials obtained by nVCA have no discriminability.
On the other hand, the result of VCA is mostly comparable with that of our proposed method despite no justification of discriminative information. 
However, it is known that VCA increases the feature dimension (the number of basis polynomials), while normalized VCA (e.g., nVCA, our proposed method) reduces it (See \citep{kera2021monomial} for the reduction of feature dimensions).

\begin{table}[t]
    \caption{Anomaly detection accuracy (AUC score $\times 100$) in MNIST and FashionMNIST in the case when a normal class is a collection of samples labeled as $0,2,4,6,8$. Here, \textit{Rot.} denotes the rotation transformations. \textit{R.A.} denotes random affine transformations. The results imply that the GOAD framework is improved by using feature extract transformation (e.g., rotation for image data). \label{table:classification_vca_nvca}}
    \begin{center}
      \begin{tabular}{c|cc|cc}
          \toprule      
        
         & \multicolumn{4}{c}{Method} \\
         Training set& \multicolumn{2}{|c|}{VCA} & \multicolumn{2}{|c}{nVCA} \\
         & \textit{Rot.} & \textit{R.A.} & \textit{Rot.} & \textit{R.A.} \\
         \midrule
          MNIST &  \textbf{76.1} & 52.4 & \underline{53.5} & 48.4 \\
          FashionMNIST & \textbf{84.5} & 51.7 & \underline{76.0} & 45.0 \\
           \bottomrule
          \end{tabular}
    \end{center}
\end{table}

\section{CONCLUSION}

In this paper, we proposed to exploit polynomials normalized on other data in the monomial-order-free basis construction of the vanishing ideal. 
The normalization on other data allows us to construct polynomials with discriminability for the normalizing data.
Throughout this paper, depending on self-supervised learning, we focused on the choice of normalizing data.
We theoretically and experimentally showed that, if normalizing data is designed outside the space where the original data belongs, then the polynomials have no discriminability in the space.
Specifically, anomaly detection was performed for the situation where no other classes can be accessed.
As a consequence, the effectiveness of our proposed method was shown.
An interesting future direction is to find another way to obtain the normalizing set.
It is also important to design a more scalable algorithm to compute basis polynomials.

\bibliographystyle{abbrvnat}
\bibliography{ref_arxiv}
\appendix
\onecolumn

\theoremstyle{plain}
\newtheorem{theoremsupp}{Theorem}[section]
\newtheorem{lemmasupp}[theoremsupp]{Lemma}
\newtheorem{corollarysupp}[theoremsupp]{Corollary}
\newtheorem{conjecturesupp}[theoremsupp]{Conjecture}
\newtheorem{propositionsupp}[theoremsupp]{Proposition}

\theoremstyle{definition}
\newtheorem{definitionsupp}[theoremsupp]{Definition}
\newtheorem{notationsupp}[theoremsupp]{Notation}
\newtheorem{examplesupp}[theoremsupp]{Example}
\newtheorem{remarksupp}[theoremsupp]{Remark}
\newtheorem{problemsupp}[theoremsupp]{Problem}
\newtheorem{algorithmsupp}[theoremsupp]{Algorithm}

\section{OMITTED PROOFS}

\subsection{Proof of Proposition~5}
In this section, we present the detailed proof of Proposition~5.
Before we prove Proposition~5, we start with the following lemmas.

\begin{lemmasupp}\label{supplementlem:chi_concentration}
  Let $Z\sim \chi_{M}^2$ be a random variable, where $\chi_{M}^2$ is the chi-square distribution with $M$ degrees of freedom.
  Then, for all $\epsilon>0$ we have
  \begin{align*}
      \mathrm{P}[Z\leq (1-\epsilon)M]\leq e^{-\epsilon^2M/6},      
  \end{align*}
  and we have
  \begin{align*}
      \mathrm{P}[Z\geq (1+\epsilon)M]\leq e^{-\epsilon M \lambda},
  \end{align*}
  where $\lambda \in (0,1/2)$.
\end{lemmasupp}
\begin{proof}
By using the argument of Lemma~B.12 in \citep{shalev2014}, we can immediately prove our statement. 
Indeed, before one sets $\lambda=\epsilon/6$ in the argument of Lemma~B.12 in \citep{shalev2014}, the desired inequalities are obtained
\end{proof}
\begin{corollarysupp}\label{supplementcor:chi_concentration}
  Let $Z\sim \chi_{M}^2$.
  Then, for all $\epsilon>0$ we have
  \begin{align*}
      \mathrm{P}\left[\left\vert \frac{Z}{M}-1 \right\vert \geq \epsilon\right]\leq e^{-\epsilon^2M/6}+e^{-\epsilon M /3}.      
  \end{align*}
\end{corollarysupp}
\begin{proof}
  Our statement is immediate if $\lambda=1/3$ in the second inequality in Lemma~\ref{supplementlem:chi_concentration}.
\end{proof}

\begin{lemmasupp}\label{supplementlem:inner_product_concentration}
  Let $X_1,\ldots,X_M, Y_1,\ldots, Y_M$ be $2M$ independent normally distributed random variables.
  Put $Z=X_1Y_1+\cdots +X_MY_M$.
  Then, for all $\epsilon\in(0,1)$ we have
  \begin{align*}
      \mathrm{P}\left[ \frac{1}{M} \left\vert Z \right\vert \geq \epsilon  \right] \leq 2e^{-\epsilon^2M/2}.
  \end{align*}
\end{lemmasupp}

\begin{proof}
  In order to prove our statement, it is enough to prove the two following inequalities:
  \begin{align}\label{supplementeq:chi1}
      \mathrm{P}\left[ Z/M  \geq \epsilon  \right] \leq e^{-\epsilon^2M/2} \quad \mathrm{and} \quad
    \mathrm{P}\left[ Z/M  \leq -\epsilon  \right] \leq e^{-\epsilon^2M/2}.
  \end{align}
  To prove both bounds, we use Chernoff's bounding method.

  (a) Proof of the first inequality of (\ref{supplementeq:chi1}).
  We first compute $\mathrm{E}\left[ e^{\lambda X_1Y_1}\right]$. 
  Since $\lambda \in (0,1)$ and $X_1,Y_1$ are independent normally distributed, we have
  \begin{align*}
      \begin{aligned}
          \mathrm{E}\left[ e^{\lambda X_1Y_1}\right]
          &=
          \frac{1}{2\pi}\int_{-\infty}^{\infty}\int_{-\infty}^{\infty} e^{\lambda xy}e^{\frac{-x^2}{2}}e^{\frac{-y^2}{2}} dx dy\\
          &=
          \frac{1}{2\pi}\int_{-\infty}^{\infty}e^{\frac{-y^2}{2}}\int_{-\infty}^{\infty} e^{\frac{-1}{2}(x-\lambda y)^2+\frac{1}{2}\lambda^2y^2} dx dy\\
          &=
          \frac{1}{\sqrt{2\pi}}\int_{-\infty}^{\infty}e^{\frac{-y^2}{2}} e^{\frac{1}{2}\lambda^2y^2} dy\\
          &=
          \frac{1}{\sqrt{1-\lambda^2}}
          \frac{1}{\sqrt{2\pi\frac{1}{1-\lambda^2}}}\int_{-\infty}^{\infty}e^{\left(\frac{-y^2}{2/(1-\lambda^2)}\right)} dy\\
          &=
          \frac{1}{\sqrt{1-\lambda^2}}.
      \end{aligned}
  \end{align*}

  Let $\epsilon\in (0,1)$.
  Applying Chernoff's bounding method we get that
  \begin{align*}
      \begin{aligned}
          \mathrm{P}\left[  Z/M  \geq \epsilon  \right]
          & =
          \mathrm{P}\left[  Z  \geq \epsilon M \right]\\
          & = 
          \mathrm{P}\left[  e^{\lambda Z}  \geq e^{\lambda\epsilon M} \right]\\
          & \leq 
          e^{-\lambda \epsilon M}\mathrm{E}\left[  e^{\lambda Z}\right]\\
          & = 
          e^{-\lambda \epsilon M}(1-\lambda^2)^{-M/2}\\
          & \leq
          e^{-\lambda \epsilon M}e^{\lambda^2M/2}, 
      \end{aligned}
  \end{align*}
  where the last inequality occurs because $(1-a)\leq e^{-a}$ for $a\geq 0$.
  Setting $\lambda=\epsilon$, we obtain the first inequality of (\ref{supplementeq:chi1}).

  (b) Proof of the second inequality of (\ref{supplementeq:chi1}). 
  We can similarly compute 
  \begin{align*}
      (\ast) \ \mathrm{E}\left[ e^{-\lambda X_1Y_1}\right]=(1-\lambda^2)^{-1/2},   
  \end{align*}
  where $\lambda\in(0,1)$.
  Also, applying Chernoff's bounding method, we have 
  \begin{align*}
      \mathrm{P}\left[  Z/M  \leq -\epsilon  \right]=\mathrm{P}\left[  -Z  \geq \epsilon M \right]\leq  e^{-\lambda \epsilon M}\mathrm{E}\left[  e^{-\lambda Z}\right].
  \end{align*}
  Using the above inequality and the equality $(\ast)$, we can prove the desired inequality similarly to (a).
\end{proof}

\begin{lemmasupp}\label{supplementlem:max_norm_concentration}
  Let $a_{i,j}$ $(1\leq i \leq M, \ 1\leq j \leq N)$ be $MN$ independent normally distributed variables.
  Put $A = [a_{i,j}]_{1\leq i \leq M, \ 1\leq j \leq N}$.
  Then, for all $\epsilon\in(0,1)$ we have
  \begin{align*}
    \begin{aligned}
      \mathrm{P}\left[\left\Vert \frac{1}{M}A^TA-I_N \right\Vert _{\mathrm{max}}\geq \epsilon\right]
      &\leq (N^2-N)e^{-\epsilon^2 M/2}
      +N(e^{-\epsilon^2 M/6}+e^{-\epsilon M/3}),    
    \end{aligned}
  \end{align*}
  where $\Vert X \Vert _{\mathrm{max}}$ means the max norm of a matrix $X$ and $I_N$ is the identity matrix of size $N$.
\end{lemmasupp}
\begin{proof}
  Put $\bm{a}_i=(a_{1,i},\ldots,a_{M,i})^T$.
  The $(i,j)$-th entry of $\frac{1}{M}A^TA-I_N$ is of the form $\bm{a}_i^T\bm{a}_j/M-\delta_{ij}$, where $\delta_{ij}$ is the Kronecker delta.
  Hence, we have 
  \begin{align*}
      \begin{aligned}
          &\mathrm{P}\left[\left\Vert \frac{1}{M}A^TA-I_N \right\Vert _{\mathrm{max}}\geq \epsilon\right]\\
          & \leq \sum_{1\leq i\leq j\leq N}\mathrm{P}\left[\left\vert \frac{\bm{a}_i^T\bm{a}_j}{M}-\delta_{ij} \right\vert \geq \epsilon\right]\\
          &= 
          N\mathrm{P}\left[\left\vert \frac{\left\Vert \bm{a}_i \right\Vert^2}{M}- 1 \right\vert \geq \epsilon\right]
          +\frac{(N^2-N)}{2}\mathrm{P}\left[\left\vert \frac{\bm{a}_1^T\bm{a}_2}{M}-\delta_{ij} \right\vert \geq \epsilon\right]
      \end{aligned}       
  \end{align*}
  By Corollary~\ref{supplementcor:chi_concentration} and Lemma~\ref{supplementlem:inner_product_concentration}, the statement is proven.
\end{proof}

\begin{lemmasupp}\label{supplementlem:triangle}
  Let $C\in \mathbb{R}^{N\times N}$ be a matrix and let $\bm{b}\in \mathbb{R}^N$ be a vector.
  If $\left\vert \Vert\bm{b} \Vert^2-1 \right\vert<\epsilon$ and $\Vert C-I_N \Vert _{\mathrm{max}}< \frac{\epsilon}{(1+\epsilon)N}$, then $\vert \bm{b}^T C \bm{b}-1 \vert<2\epsilon$.
\end{lemmasupp}
\begin{proof}
  We first remark that
  \begin{align*}
      \ \left\vert \bm{b}^T (C-I_N)\bm{b} \right\vert\leq \Vert C-I_N \Vert _{\mathrm{max}} (\Vert \bm{b}\Vert _1)^2 \quad \mathrm{and} \quad  
      \ \Vert \bm{b}\Vert _1\leq  \sqrt{N}\Vert \bm{b}\Vert,      
  \end{align*}
  where $\Vert \bm{x}\Vert_1$ means the L1 norm of a vector $\bm{x}$.
  If $\left\vert \Vert\bm{b} \Vert^2-1 \right\vert<\epsilon$ and $\Vert C-I_N \Vert _{\mathrm{max}}< \frac{\epsilon}{(1+\epsilon)N}$, then we have
  \begin{align*}
      \begin{aligned}
          \left\vert \bm{b}^T (C-I_N)\bm{b} \right\vert
          &\leq 
          \Vert C-I_N \Vert _{\mathrm{max}} (\Vert \bm{b}\Vert _1)^2\\
          &< \frac{\epsilon}{(1+\epsilon)N} (\Vert \bm{b}\Vert _1)^2\\
          &\leq
          \frac{\epsilon}{1+\epsilon}\Vert \bm{b}\Vert^2\\
          &\leq \epsilon.
      \end{aligned}
  \end{align*}
  Hence, by the triangle inequality, we obtain
  \begin{align*}
      \begin{aligned}
          \left\vert \bm{b}^T C\bm{b}-1 \right\vert&\leq 
          \left\vert \bm{b}^T C\bm{b}-\bm{b}^T\bm{b} \right\vert + \left\vert \bm{b}^T\bm{b}-1 \right\vert\\
          &=
          \left\vert \bm{b}^T (C-I_N)\bm{b} \right\vert + \left\vert \Vert\bm{b}\Vert^2-1 \right\vert<2\epsilon.
      \end{aligned}
  \end{align*}
\end{proof}

Now we prove Proposition~5.

\begin{propositionsupp}[Proposition~5]\label{supplementprop:concentration}
  Let $X\subset \mathbb{R}^n$ be a set of points and let $\{\bm{u}_1,\ldots,\bm{u}_n\}$ be an orthogonal basis of $\mathbb{R}^n$ such that $\Span{X}=\Span{\bm{u}_1,\ldots,\bm{u}_k}$.
  Choose a set, $Z$, of random points satisfying that any point $\bm{p}_i\in Z$ is of the form $\bm{p}_i=\sum_{j=1}^na_{i,j}\bm{u}_j$, where $a_{i,j}$'s $( 1\leq i \leq |Z|, 1\leq j\leq n)$ are i.i.d. $\sim N(0,1)$.
  If $\bm{w}=\sum_{i=k+1}^n w_i\bm{u}_{i}$ is a random vector such that $w_i$'s are i.i.d. $\sim N(0,1/(n-k))$, 
  then we have 
  \begin{align}\label{supplementeq:prop_statement}
      \begin{aligned}
          \mathrm{P}\left[ \lVert g(X) \rVert =0 \right] &=1,\\
          \mathrm{P}\left[\left\vert \frac{\lVert g(Z)\rVert^2}{\vert Z \vert}-1 \right\vert \geq 2\epsilon\right]
          &\leq 
           e^{-\epsilon^2N/6}+e^{-\epsilon N /3}
           +(N^2-N)e^{-\eta^2 |Z|/2}
           +N(e^{-\eta^2 |Z|/6}+e^{-\eta |Z|/3}),
      \end{aligned}
  \end{align}
  where $g(\bm{x})=\bm{w}^T\bm{x}$, $N=n-k$ and $\eta=\frac{\epsilon}{(1+\epsilon)N}$.
\end{propositionsupp}
\begin{proof}
We first prove the first equality of (\ref{supplementeq:prop_statement}).
As $\bm{w}\perp \Span{X}=\Span{\bm{u}_1,\ldots,\bm{u}_k}$, we have $g(\bm{x})=0$ for any $\bm{x}\in X$.
Therefore, $\mathrm{P}\left[ \lVert g(X) \rVert =0 \right] =1$.

We next prove the second inequality  of (\ref{supplementeq:prop_statement}).
Putting $A=[a_{i,j}]_{1\leq i \leq |Z|, k+1\leq j\leq n}$ and $\bm{b}=(w_{k+1},\ldots,w_n)^T$, we have
\begin{align*}
  \Vert g(Z) \Vert^2=\sum_{i=1}^{|Z|} (\bm{w}^T\bm{p}_i)^2=\bm{b}^TA^TA\bm{b}.
\end{align*}
By Lemma~\ref{supplementlem:triangle}, we obtain
\begin{align*}
  \begin{aligned}
      \mathrm{P}\left[ \left\vert \frac{\sum_{i=1}^{|Z|} (\bm{w}^T\bm{p}_i)^2}{|Z|}-1 \right\vert \geq 2\epsilon\right]
      &\leq 
      \mathrm{P}\left[\left\Vert \frac{1}{|Z|}A^TA-I_N \right\Vert _{\mathrm{max}}\geq \eta\right]
       +
      \mathrm{P}\left[ \vert \Vert \bm{w} \Vert^2-1 \vert \geq \epsilon \right],
  \end{aligned}
\end{align*}
where $\eta=\frac{\epsilon}{(1+\epsilon)N}$.
Hence, the second inequality  of (\ref{supplementeq:prop_statement}) is immediate if the following inequalities are proven:
\begin{align}\label{supplementeq:prop_proof}
  \begin{aligned}
      \mathrm{P}\left[\left\Vert \frac{1}{|Z|}A^TA-I_N \right\Vert _{\mathrm{max}}\geq \eta\right]
      &\leq(N^2-N)e^{-\eta^2 |Z|/2} 
      +N(e^{-\eta^2 |Z|/6}+e^{-\eta |Z|/3}),\\  
      \mathrm{P}\left[\left\vert \Vert \bm{w} \Vert^2-1 \right\vert \geq \epsilon\right]
      &\leq e^{-\epsilon^2N/6}+e^{-\epsilon N /3}.
  \end{aligned}
\end{align}

(a) Proof of the first inequality of (\ref{supplementeq:prop_proof}).
Sinece $0<\eta<1$, by Lemma~\ref{supplementlem:max_norm_concentration}, we obtain the first inequality of (\ref{supplementeq:prop_proof}).

(b) Proof of the second inequality of (\ref{supplementeq:prop_proof}).
Since $\sqrt{n-k}w_i$'s are i.i.d.$\sim N(0,1)$, 
we have 
\begin{align*}
    N\Vert b\Vert^2=\sum_{i=k+1}^n(\sqrt{n-k}w_i)^2\sim\chi_N^2.
\end{align*}
By Corollary~\ref{supplementcor:chi_concentration}, the desired inequality holds.
\end{proof}

\subsection{Proof of Theorem~6}
In this section, we prove Theorem~6.
\begin{theoremsupp}[Theorem~6]
  Let $X,Y\subset \mathbb{R}^n$ be sets of points and we denote $V_X=\Span{X}$ and $V_Y=\Span{Y}$.
  Let $\{\bm{u}_1,\ldots,\bm{u}_n\}$ be an orthogonal basis of $\mathbb{R}^n$ such that 
  \begin{align*}
      V_X=\Span{\bm{u}_1,\ldots,\bm{u}_k} \ \mathrm{and} \ 
      V_X+V_Y=\Span{\bm{u}_1,\ldots,\bm{u}_{k+m}}.      
  \end{align*}
  Let $\bm{w}=\sum_{i=k+1}^n w_i\bm{u}_{i}$ be a random vector as in Proposition~\ref{supplementprop:concentration}.
  Choose random vectors $\bm{p}_X \in V_X$, $\bm{p}_Y \in V_Y$ and $\bm{p}_{\perp}\in (V_X+V_Y)^{\perp}$ such that coefficients of $\bm{u}_1,\ldots,\bm{u}_n$ are i.i.d.$\sim N(0,1)$.
  Then, for $t>0$, we have 
  \begin{align}\label{supplementeq:theorem_statement}
      \begin{aligned}
          \mathrm{P}\left[ \left\vert g(\bm{p}_X)\right\vert=0\right] &= 1,\\
          \mathrm{P}\left[\left\vert g(\bm{p}_Y)\right\vert<t\right] &\geq 1-\frac{m}{n-k}\cdot\frac{1}{t^2},\\
          \mathrm{P}\left[\left\vert g(\bm{p}_{\perp})\right\vert\geq \sqrt{\frac{n-(k+m)}{2(n-k)}}\right] &\geq \frac{n-(k+m)}{4(n-(k+m)+2)}. 
      \end{aligned}
  \end{align}
\end{theoremsupp}
\begin{proof}
(a) Let us first prove the first equality of (\ref{supplementeq:theorem_statement}).
Since $\bm{w}\perp \Span{X}=\Span{\bm{u}_1,\ldots,\bm{u}_k}$, we have $g(\bm{p}_x)=0$.
Therefore, $\mathrm{P}\left[ \lvert g(\bm{p}_X) \rvert =0 \right] =1$.

(b)
For the second inequality of (\ref{supplementeq:theorem_statement}), we use the Chebyshev inequality.
In order to use the Chebyshev inequality, we need to compute $\mathrm{E}[g(\bm{p_Y})]$ and $\mathrm{Var}[g(\bm{p_Y})]$.

Let $\bm{p}_Y=\sum_{i=1}^{k+m}a_i\bm{u}_i$, where $a_1,\ldots,a_{k+m}$ are i.i.d.$\sim N(0,1)$.
Since $\bm{w}\perp V_X=\Span{\bm{u}_1,\ldots,\bm{u}_k}$, we have $g(\bm{p_Y})=\bm{w}^T\bm{p_Y}=\sum_{i=k+1}^{k+m}a_iw_i$.

As $a_i$ and $w_i$ are independent, we obtain
\begin{align*}
  \mathrm{E}[ a_iw_i]=0 \ \mathrm{and} \ \mathrm{Var}[a_iw_i] = \frac{1}{n-k}.
\end{align*}
Also since $a_1w_1, \ldots, a_nw_n$ are independent, we have
\begin{align*}
  \begin{aligned}
      \mathrm{E}[g(\bm{p_Y})] &= \sum_{i=k+1}^{k+m}\mathrm{E}[a_iw_i] = 0,\\
      \mathrm{Var}[g(\bm{p_Y})] &= \sum_{i=k+1}^{k+m} \mathrm{Var}\left[ a_iw_i\right]=\frac{m}{n-k}.
  \end{aligned}
\end{align*}
Therefore, by the Chebyshev inequality, we have
\begin{align*}
  \mathrm{P}\left[\left\vert g(\bm{p}_Y)\right\vert<t\right] &\geq 1-\frac{m}{n-k}\cdot\frac{1}{t^2}.
\end{align*}

(c)
For the third inequality of (\ref{supplementeq:theorem_statement}), we use the Paley-Zygmund inequality.
In order to use the Paley--Zygmund inequality, we need to compute $\mathrm{E}\left[g(\bm{p}_{\perp})^2\right]$ and $\mathrm{E}\left[g(\bm{p}_{\perp})^4\right]$.

Let $\bm{p}_{\perp}=\sum_{i=1}^{n}a_i\bm{u}_i$, where $a_{1},\ldots,a_{n}$ are i.i.d.$\sim N(0,1)$.
Since, $\bm{w}\perp V_X=\Span{\bm{u}_1,\ldots,\bm{u}_k}$, we have $g(\bm{p_{\perp}})=\bm{w}^T\bm{p_{\perp}}=\sum_{i=k+1}^{n}a_iw_i$.
By a similar way as the above case, we have
\begin{align*}
  \begin{aligned}
    \mathrm{E}\left[g(\bm{p}_{\perp})^2\right] 
    &=
    \mathrm{E}\left[g(\bm{p}_{\perp})^2\right]-\left(\mathrm{E}\left[g(\bm{p}_{\perp})\right]\right)^2\\
    &=
    \mathrm{Var}\left[\bm{w}^T\bm{p}_{\perp}\right] = \frac{n-(k+m)}{n-k}.      
  \end{aligned}
\end{align*}

We next compute $\mathrm{E}\left[g(\bm{p}_{\perp})^4\right]$.

Remarking that $a_iw_i$ and $a_jw_j$ $(i\neq j)$ are independent variables with means $0$, we have
\begin{align*}
  \begin{aligned}
      \mathrm{E}\left[g(\bm{p}_{\perp})^4\right]
      &=\mathrm{E}[(\bm{w}^T\bm{p}_{\perp})^4] \\
      & =\mathrm{E}\left[ \left( \sum_{i=k+m+1}^n a_iw_i \right)^4 \right]\\
      &= \sum_{(i_1,i_2,i_3,i_4)\in \{k+m+1,\ldots,n\}^4}\mathrm{E}_{i_1,i_2,i_3,i_4},
  \end{aligned}
\end{align*}
where $\mathrm{E}_{i_1,i_2,i_3,i_4}=\mathrm{E}[a_{i_1}a_{i_1}a_{i_3}a_{i_4}w_{i_1}w_{i_2}w_{i_3}w_{i_4}]$.
We remark that 
\begin{align*}
  \mathrm{E}_{i_1,i_2,i_3,i_4}=
  \left\{
  \begin{array}{ll}
      9/(n-k)^2 & \text{if $i_1=i_2=i_3=i_4$}\\
      1/(n-k)^2 & \text{if $i_1 = i_2 \neq i_3=i_4$}\\
      1/(n-k)^2 & \text{if $i_1 = i_3 \neq i_2=i_4$}\\
      1/(n-k)^2 & \text{if $i_1 = i_4 \neq i_2=i_3$}\\
      0 & \text{otherwise}.
  \end{array}\right.
\end{align*}
and we have
\begin{align*}
  \begin{aligned}
      \mathrm{E}\left[g(\bm{p}_{\perp})^4\right] 
      & = \frac{9}{(n-k)^2}(n-k-m)
       +\frac{1}{(n-k)^2}3(n-k-m)(n-k-m-1)\\
      & = \frac{3(n-k-m)(n-k-m+2)}{(n-k)^2}.
  \end{aligned}
\end{align*}
By the Paley-Zygmund inequality, the following holds:
\begin{align*}
  \begin{aligned}
      \mathrm{P}\left[ |g(\bm{p}_{\perp})|>\sqrt{\theta \mathrm{E}[g(\bm{p}_{\perp})^2]} \right]
      &= \mathrm{P}\left[ g(\bm{p}_{\perp})^2>\theta \mathrm{E}\left[g(\bm{p}_{\perp})^2\right] \right]\\
      & \geq (1-\theta)^2 \frac{(\mathrm{E}\left[g(\bm{p}_{\perp})^2\right])^2}{\mathrm{E}[g(\bm{p}_{\perp})^4]}\\
      & = (1-\theta)^2 \frac{n-(k+m)}{n-(k+m)+2},
  \end{aligned}
\end{align*}
where $\theta \in [0,1]$.
Setting $\theta=1/2$, we have the third inequality of (\ref{supplementeq:theorem_statement}).
\end{proof}

\subsection{Proof of Theorem~17}

In this section, we present the detailed proof of Theorem~17.

\subsubsection{Basic Definitions and Notations for Ideals}
We here some basic facts and terminology about ideals.
Let $k$ be a field and let $k[x_1,\ldots,x_n]$ be a polynomial ring, where $x_i$ is the $i$-th indeterminate.
We assume that ideals are defined in $k[x_1,\ldots,x_n]$, unless otherwise stated.

\begin{definitionsupp}
    An ideal $I$ is \emph{radical} if $f^m\in I$ for some integer $m\geq 1$ implies that $f\in I$.

    Let $I$ be an ideal.
    The  radical ideal, $\sqrt{I}$,  of $I$  is the set 
    \begin{align*}
        \{f\mid f^m\in I \text{  {\rm for some integer} } m\geq 1\}.
    \end{align*}
\end{definitionsupp}
\begin{remarksupp}
    A radical ideal is an ideal.
\end{remarksupp}
\begin{definitionsupp}
    Let $G\subset k[x_1,\ldots,x_n]$.
    Then we set
    \begin{align*}
      V(G)=\{\bm{x}\in k^n\mid f(\bm{x})=0 \ \text{{\rm for all}} \ f\in G\}.
    \end{align*} 
    We call $V(G)$ the \emph{algebraic set} defined by $G$ over $k$.
    When we emphasize a field $k$, we denote $V(G)$ by $V_k(G)$.
\end{definitionsupp}

\begin{definitionsupp}    
    An algebraic set $V\subset k^n$ is \emph{irreducible} if $V=V_1\cup V_2$, where $V_1$ and $V_2$ are algebraic sets over $k$, then $V_1=V$ or $V_2=V$.
\end{definitionsupp}

\begin{definitionsupp}
    Let $V\subset k^n$ be an algebraic set.
    A decomposition
    \begin{align*}
        V=V_1\cup\cdots\cup V_s,
    \end{align*}
    where each $V_i$ is an irreducible algebraic set, is called a \emph{minimal decomposition} if $V_i\not\subset V_j$ for $i\neq j$.
    Also, we call the $V_i$ the \it{irreducible components} of $V_i$.
\end{definitionsupp}

\begin{definitionsupp}
    Let $V\subset k^n$ be an irreducible algebraic set.
    We define 
    \begin{align*}
        \dim V = \sup \{r \mid V_0\subsetneq V_1\subsetneq \cdots \subsetneq V_r=V, \ V_i \text{: an irreducible algebraic set over $k$}\}.
    \end{align*}
    We call $\dim V$ the \emph{dimension} of $V$.
\end{definitionsupp}
\begin{remark}
    It is well-known that the dimension of an irreducible algebraic set is finite.
\end{remark}

\begin{definitionsupp}
    Let $X\subset k^n$ be a subset of $k^n$. 
    Then we set 
    \begin{align*}
        \mathcal{I}(X)=\{g\in k[x_1,\ldots,x_n]\mid g(\bm{x})=0 \text{ {\rm for all} } \bm{x}\in X\}.
    \end{align*}
    When we emphasize a field $k$, we denote $\videal X$ by $\mathcal{I}_k(X)$.
\end{definitionsupp}

\begin{definitionsupp}
    Let $S\subset k^n$.
    We define $\overline{S}=V(\mathcal{I}(S))$.
\end{definitionsupp}

\begin{definitionsupp}
    Let $I$ and $J$ be ideals.
    Then we set 
    \begin{align*}
        I:J=\{g\in k[x_1,\ldots,x_n]\mid gJ \subset I\}.
    \end{align*}
    We call $I:J$ the \emph{ideal quotient} of $I$ by $J$.
\end{definitionsupp}
\begin{remarksupp}
    An ideal quotient is an ideal.
\end{remarksupp}

The following facts are well-known.
Our main reference is \citep{cox2015ideals}.
\begin{lemmasupp}\label{supplementlem:properties_of_closure}
    Let $S$ and $T$ be subsets of $k^n$.
    Then we have
    \begin{align*}
        (i) \ \videal S=\videal {\overline{S}} \ \mathrm{and} \ (ii) \ \overline{S\cup T}=\overline{S}\cup\overline{T}.
    \end{align*}
\end{lemmasupp}
\begin{theoremsupp}[{\cite[Ch.\,4 Sect.\,6 Theorem\,4]{cox2015ideals}}]
  Let $V\subset k^n$ be an algebraic set.
  Then, $V$ has a minimal decomposition 
  \begin{align*}
      V=V_1\cup\cdots\cup V_s.
  \end{align*}
  Furthermore, this minimal decomposition is unique up to the order in which $V_1,\ldots,V_s$ are written.
\end{theoremsupp}

\begin{lemmasupp}\label{supplementlem:irr.subset_irr.comp}
    Let $V, V_1,V_2\ldots,V_s$ be algebraic sets.
    If $V$ is irreducible and $V\subset V_1\cup\cdots\cup V_s$, then $V\subset V_i$ for some $i$.
\end{lemmasupp}


\begin{propositionsupp}[{\cite[Ch.\,4 Sect.\,4 Corollary\,11]{cox2015ideals}}]\label{supplementprop:quotient}
    Let $V$ and $W$ be algebraic sets over $k$.
    Then we have $\videal V:\videal W=\videal {V\setminus W}$.
\end{propositionsupp}

\begin{theoremsupp}[The Strong Nullstellensatz, {\cite[Ch\,4 Sect.\,2 Theorem\,6]{cox2015ideals}}]
    Let $G\subset k[x_1,\ldots,x_n]$. 
    If $k$ is algebraically closed, then $\videal {V(G)}=\sqrt{\langle G \rangle}$.
\end{theoremsupp}

\begin{lemmasupp}\label{supplementlem:dim}
    Let $V,W\subset k^n$ be irreducible algebraic sets.
    If $\dim W=\dim V$ and $V\subset W$, then $V=W$.
\end{lemmasupp}

\begin{lemmasupp}\label{supplementlem:irr_minus}
    Let $V$ be an irreducible algebraic set and let $W$ be an algebraic set.
    If $W\not\subset V$, then $\overline{V\setminus W}=V$.
\end{lemmasupp}

\subsubsection{An Ideal Given by \texttt{VCA}(X,Y)}
In this section, we prove Theorem~7.
In the following, we consider the case when $k=\mathbb{R}$ or $\mathbb{C}$.

\begin{definitionsupp}
    Let $I$ be an ideal in $\mathcal{R}_n$.
    Then we define an ideal $I_{\mathbb{C}}$ in $\mathbb{C}[x_1,\ldots,x_n]$ as follows:
    \begin{align*}
        I_{\mathbb{C}}=\left\{ \sum_{i=1}^m f_i g_i \mid f_i\in\mathbb{C}[x_1,\ldots,x_n], g_i\in I  \right\}.
    \end{align*} 
\end{definitionsupp}

\begin{lemmasupp}\label{supplementlem:sqrt_complex}
    Let $G\subset\mathcal{R}_n$.
    Then we have $\sqrt{\langle G\rangle_{\mathbb{C}}}=(\sqrt{\langle G\rangle})_{\mathbb{C}}$.
\end{lemmasupp}
\begin{proof}
    By the definitions of $\sqrt{\,\cdot\,}$ and  $(\,\cdot\,)_{\mathbb{C}}$, we can prove the statement easily.
\end{proof}

\begin{theoremsupp}[Theorem~17]
    Let $X,Y\subset\mathbb{R}^n$ be distinct point sets and let $G$ be a polynomial set, which is the output of $\texttt{VCA}(X,Y)$ for $\epsilon=0$.
    We put $V=V_{k}(G)$ and denote its irreducible components by $V_1,\ldots,V_s$.
    Then, for any irreducible algebraic set $W\subset k^n$ satisfying $\dim W= \min \dim V_i$ and $Y\subset W$, we have
    \begin{itemize}
        \item[(1)]
        $V\not\subset W$ and $W\not\subset V$.
        \item[(2)]
        $V=\overline{V\setminus W}$.
        \item[(3)]
        $\mathcal{I}_{k}(V)=\mathcal{I}_{k}(V\setminus W)=\mathcal{I}_{k}(V): \mathcal{I}_{k}(W)$.
    \end{itemize}
    
    Moreover, if $k=\mathbb{C}$, then we have
    \begin{itemize}
        \item[(3)']
        $\left(\sqrt{\langle G\rangle}\right)_{\mathbb{C}}=\videalc V=\videalc {V\setminus W}=\videalc V:\videalc W$.
    \end{itemize}
\end{theoremsupp}
\begin{proof}
(1):
We first prove $V\not\subset W$.
If we assume $V\subset W$, then we have $V_i\subset W$ for all $i$.
Also, by the assumption, $\dim V_{i_0}=\dim W$ for some $i_0$.
Since $V_{i_0}\subset W$, by Lemma~\ref{supplementlem:dim}, $V_{i_0}=W$.
Hence, we have $Y\subset W=V_{i_0}\subset V$.
This leads us to a contradiction.
Therefore, $V\not\subset W$.

We next prove $W\not\subset V$.
If we assume that $W\subset V$, then there exits $i_0$ such that $W\subset V_{i_0}$ by Lemma~\ref{supplementlem:irr.subset_irr.comp}.
Hence, we have $Y\subset W\subset V_{i_0}\subset V$.
This also leads us to a contradiction.

(2) and (3):
Before we prove (2) and (3), we start with the following claims.

\underline{Claim}: $Y\not\subset V$.

Proof of Claim.
We assume $Y\subset V=V_{k}(G)$.
This means that $g(Y)=\bm{0}$ for $g\in G$.
However, this is impossible as $g$ is nonvanishing for $Y$.
Therefore, $Y\not\subset V$.

\underline{Claim}: $V_i\not\subset W$ for all $V_i$.

Proof of Claim.
If we assume $V_i\subset W$ for some $V_i$, then $\dim V_i \leq \dim W$ holds.
By the condition of the dimension of $W$, we have $\dim W = \dim V_i$.
By Lemma~\ref{supplementlem:dim}, $W=V_i$.
Hence, we have $Y\subset W=V_i\subset V$.
This leads us to a contradiction.

Now we go back to prove (2) and (3).
Using Lemmas~\ref{supplementlem:properties_of_closure} and \ref{supplementlem:irr_minus}, we get that
\begin{align*}
    \begin{aligned}
        \overline{V\setminus W}
        &=\overline{(V_1\cup\cdots\cup V_s)\setminus W}\\
        &=\overline{V_1\setminus W}\cup\cdots\cup \overline{V_s\setminus W}\\
        &=V_1\cup\cdots\cup V_s\\
        &=V
    \end{aligned}
\end{align*}
Hence, by Lemma~\ref{supplementlem:properties_of_closure} and Proposition~\ref{supplementprop:quotient}, we have
\begin{align*}
    \mathcal{I}_{k}(V)=\mathcal{I}_{k}(\overline{V\setminus W})=\mathcal{I}_{k}(V\setminus W)=\mathcal{I}_{k}(V):\mathcal{I}_{k}(W)
\end{align*}

(3)':
The statement is proven by the Nullstellensatz and Lemma~\ref{supplementlem:sqrt_complex}.
\end{proof}



\section{ADDITIONAL EXPERIMENTS}



We experimented the proposed method under other setting.
Using three transformations $T_1$, $T_2$ and $T_3$ and the size of training sets $=$30,000, we experimented our methods of Section~6.2.
Here, when we use rotation transformations, $T_1$, $T_2$ and $T_3$ denote rotation by $0$, $90$ and $180$ degrees.
  \begin{table*}[t]
     \caption{Anomaly detection accuracy (AUC score $\times 100$) in MNIST and FashionMNIST in the case when a normal class is a collection of samples labeled as $0,2,4,6,8$. Here, \textit{Rot.} denotes to use the rotation transformations. \textit{R.A.} denotes to use random affine transformations. 
     In each experiment, $3$ transformations are used.
     In particular, rotation transformations are rotation by $0$, $90$ and $180$ degrees.
     Compared to using 4 transformations ( Table 1 of Section 6.2), using 3 transformations improves the results.
    In particular, using 3 rotation transformations for MNIST improves an anomaly score.
    It is considered that kinds of normalizing data sets in the hand-written-digit space have an impact on discriminability.
     \label{supplementtable:classification}}
    \begin{center}
      \begin{tabular}{c|c|cc}
          \toprule      
         & Size of  &\multicolumn{2}{c}{Method} \\
         Data set & training set & \multicolumn{2}{c}{Ours} \\
         & & \textit{Rot.} & \textit{R.A.}\\
         \midrule
          \multirow{1}{*}{MNIST}  & 30,000 & \textbf{84.2} & 53.8\\
           \midrule
          FashionMNIST& 30,000
          & \textbf{82.6} & 60.6\\
           \bottomrule
          \end{tabular}
    \end{center}
\end{table*}

The scores of the \textit{Rot.} column of Table~\ref{supplementtable:classification} are enhanced over those of Table~1 of Section 6.2.
In particular, when we use the MNIST sets, the score growth is better when the number of transformations is three.
It would also be interesting to note that, depending on the number of transformations, the training of our method is different.
We have discussed the choice of normalizing datasets and stated that they should be chosen from the hand-written-digit space.
Furthermore, based on the results of this experiment, kinds of normalizing data sets in the hand-written-digit space are expected to have an impact on discriminability, but this is beyond the scope of this paper.

\end{document}